\documentclass[twoside,11pt]{article}
\usepackage{isipta}

\usepackage[utf8]{inputenc}
\inputencoding{utf8} 

\usepackage{hyperref,color,soul,booktabs}
\setulcolor{blue}
\newcommand{\BibTeX}{\textsc{B\kern-0.1emi\kern-0.017emb}\kern-0.15em\TeX}

\usepackage{tikz}
\usepackage{amsmath,amsfonts}
\usepackage{mathtools}
\usepackage{mathptmx} 
\usepackage{amssymb}
\usepackage{MnSymbol}
\usepackage{selnolig} 
\usepackage{xargs} 

\usepackage{nicefrac}

\usepackage{enumitem}

\newcommand*{\nats}{\mathbb{N}}
\newcommand*{\nonnegints}{\mathbb{Z}_{\geq0}}
\newcommand*{\reals}{\mathbb{R}}
\newcommand*{\posreals}{\mathbb{R}_{>0}}

\renewcommand*{\iff}{\Leftrightarrow}
\newcommand*{\then}{\Rightarrow}
\newcommand*{\id}[1][]{\mathrm{id}_{#1}}

\newcommand*{\vs}{\mathcal{V}}
\newcommand*{\vstoo}{\mathcal{W}}
\newcommand*{\bij}[1][]{\phi_{#1}}

\newcommand*{\vect}[1][u]{#1}

\newcommand*{\vo}[1][]{\preceq_{#1}}

\newcommand*{\svo}[1][]{\prec_{#1}}

\newcommand*{\vspos}[1][\svoi]{\vs_{{#1}0}}
\newcommand*{\vsneg}[1][\svo]{\vs_{{#1}0}}

\newcommand*{\ps}{\mathcal{X}}
\newcommand*{\psall}{\mathcal{X}^\nats}
\newcommand*{\rv}[1][]{X_{#1}}
\newcommand*{\rval}[1][]{x_{#1}}

\newcommand*{\simplex}[1][\ps]{\Sigma_{#1}}

\newcommand*{\g}[1][f]{#1}
\newcommand*{\gambles}{\mathcal{L}}
\newcommand*{\gambleson}[1][\ps]{\mathcal{L}\group{#1}}
\newcommand*{\ind}[1][\ev]{\mathbb{I}_{#1}}

\newcommand*{\posgambleson}[1][\ps]{\gambles\group{#1}_{>0}}
\newcommand*{\neggambleson}[1][\ps]{\gambles\group{#1}_{<0}}
\newcommand*{\gamblesfs}{\bar{\gambles}\group{\psall}}

\newcommand*{\cf}{C}
\newcommand*{\cftoo}{C'}
\newcommand*{\cfdom}{\mathcal{Q}}

\newcommand*{\os}[1][]{O_{#1}}
\newcommand*{\tos}[1][]{\tilde{O}_{#1}}

\newcommand*{\soiv}[1][]{I_{#1}}
\newcommand*{\ec}[1][\vect]{\sqgroup{#1}}
\newcommand*{\qs}[1][\soiv]{\vs/{#1}}

\newcommand*{\sodv}[1][]{D_{#1}}

\newcommand*{\eqrel}[1][k]{\simeq_j}
\newcommandx{\eclexi}[2][1=f, 2=k]{{#1}/\ker\Lambda_{#2}}
\newcommandx{\qslexi}[2][1=\gambleson, 2=k]{{#1}/\ker\Lambda_{#2}}

\newcommand*{\permuts}[1][]{\mathcal{P}_{#1}}
\newcommand*{\permut}[1][]{\pi_{#1}}

\newcommand*{\lift}[1][\permut]{{#1}^t}
\newcommand*{\len}[1][]{n_{#1}}
\newcommand*{\ex}{\mathrm{inv}_{\permuts[\len]}}
\newcommand*{\pia}[1][\rval]{[#1]}
\newcommand*{\cts}{m}
\newcommand*{\cm}[1][]{T^{#1}}
\newcommand*{\cvs}[1][]{\mathcal{N}^{#1}}

\newcommand*{\pigambleson}[1][\ps]{\gambles_{{\permuts[\len]}}\group{#1}}
\newcommand*{\Hymap}[1][]{\mathrm{H}_{#1}}
\newcommand*{\Hymape}[1][]{\tilde{\mathrm{H}}_{#1}}
\newcommand*{\Hymapei}[1][]{{{\tilde{\mathrm{H}}_{#1}}}^{-1}}
\newcommandx{\Hy}[3][1={\len}, 2={\cdot}, 3={\cts}]{\Hymap[{#1}]\group{{#2}\vert{#3}}}
\newcommand*{\qsg}[1][{\soiv[{\permuts[\len]}]}]{\gambleson[\ps^{\len}]/{#1}}
\newcommand*{\tg}[1][\g]{\tilde{#1}}
\newcommand*{\cg}[1][\g]{{#1}^*}
\newcommand*{\rcf}{\tilde{C}}
\newcommand*{\osin}[1][\os]{#1/\soiv[{\permuts[\len]}]}
\newcommand*{\osind}[1][\os]{#1/\soiv}

\newcommand*{\polson}[1][\simplex]{\mathcal{V}\group{#1}}
\newcommandx{\polsond}[2][1=\simplex, 2=\len]{\mathcal{V}^{#2}\group{#1}}
\newcommand*{\pol}[1][h]{#1}
\newcommand*{\bern}[1][\cts]{B_{#1}}
\newcommand*{\gc}[1][r]{#1}
\newcommand*{\CoMnmap}[1][]{\mathrm{CoM}_{#1}}

\newcommandx{\Mn}[3][1={\len},2={\g},3={\theta}]{\mathrm{M}_{#1}\group{#2\vert#3}}
\newcommand*{\Mnmap}[1][]{\mathrm{M}_{#1}}

\newcommandx{\Mne}[3][1={\len},2={\g},3={\theta}]{\tilde{\mathrm{M}}_{#1}\group{#2\vert#3}}
\newcommand*{\Mnmape}[1][]{\tilde{\mathrm{M}}_{#1}}
\newcommand*{\Mnmapei}[1][]{\tilde{\mathrm{M}}_{#1}^{-1}}

\DeclareMathOperator*{\Posi}{posi}
\DeclareMathOperator*{\Span}{span}

\DeclareMathOperator*{\kernel}{kern}
\DeclareMathOperator*{\rng}{rng}

\newcommand*{\cset}[3][]{\set[#1]{#2:#3}}

\DeclarePairedDelimiter{\group}{(}{)}
\DeclarePairedDelimiter{\sqgroup}{[}{]}
\DeclarePairedDelimiter{\set}{\{}{\}}

\DeclarePairedDelimiter{\abs}{\vert}{\vert}

\ShortHeadings{Exchangeable choice functions}{Arthur Van Camp \& Gert de Cooman} 

\begin{document}

\title{Exchangeable choice functions}
\author{\name Arthur Van Camp \email arthur.vancamp@ugent.be\\
\name Gert de Cooman \email gert.decooman@ugent.be\\
\addr IDLab, Ghent University\\
Ghent (Belgium)}
\maketitle
\begin{abstract}
We investigate how to model exchangeability with choice functions.
Exchangeability is a structural assessment on a sequence of uncertain variables.
We show how such assessments are a special indifference assessment, and how that leads to a counterpart of de Finetti's Representation Theorem, both in a finite and a countable context.
\end{abstract}
\begin{keywords}
Exchangeability; choice functions; indifference; sets of desirable gambles; representation.
\end{keywords}

\section{Introduction}
In this paper, we study how to model exchangeability, a structural assessment for uncertainty models that is important for inference purposes, in the framework of choice functions, an interesting approach to modelling uncertainty.
This work builds on the work about exchangeability for lower previsions (see~\cite{cooman2006d}) and exchangeability for sets of desirable gambles (see \cite{cooman2010}).

Choice functions are related to the fundamental problem in decision theory: how to make a choice from within a set of available options.
In their book,~\cite{Neumann1944} provide an axiomatisation of choice based on pairwise comparison between the options.
Later on,~\cite{Rubin1987} generalised that idea and proposed a theory of choice functions based on choice between more than two elements.
One of the aspects in~\citet{Rubin1987}'s theory is that, between any pair of options, the agent either prefers one of them, or he is indifferent between them, so two options can never be incomparable.
However, for instance when the information available does not allow for a complete comparison of the options, the agent may be undecided between two options without being indifferent between them; this will for instance typically be the case when there is no relevant information available at all.
This is one of the motivations for a theory of imprecise probabilities (see~\cite{walley1991}), where incomparability and indifference are distinguished.
\cite{Kadane2004} and~\cite{seidenfeld2010} generalise the axioms in~\cite{Rubin1987} to allow for incomparability.

Exchangeability is a structural assessment on a sequence of uncertain variables.
Loosely speaking, making a judgement of exchangeability means that the order in which the variables are observed is considered irrelevant.
This irrelevancy will be modelled through an indifference assessment.
The first detailed study of exchangeability was~\cite{finetti1937}.
For a brief historical overview, we refer to Ref.~\cite[Sec.~1]{cooman2010}.

In Sec.~\ref{sec:choice_indiff}, we will recall the necessary tools for modelling indifference with choice functions.
Next, in Sec.~\ref{sec:finite exchangeability}, we will derive de Finetti-like Representation Theorems for a finite sequence that is exchangeable.
We will take this one step further in Sec.~\ref{sec:countable exchangeability}, where we consider a countable sequence and derive a representation theorem for such sequences.
Because it will be useful to compare with~\cite{cooman2010}, we will also provide representation theorems for sets of desirable gambles.

\section{Choice functions, desirability and indifference}
\label{sec:choice_indiff}
The material in this section is based on~\cite[Sec.~5]{Vancamp2017}.
Consider a real vector space $\vs$, provided with the vector addition and scalar multiplication.
Elements $\vect$ of $\vs$ are intended as abstract representations of \emph{options} amongst which a subject can express his preferences, by specifying, as we will see below, choice functions.
Mostly, options will be real-valued maps on the possibility space, interpreted as uncertain rewards, and therefore also called \emph{gambles}.
The set of all gambles on the possibility space $\ps$ will be denoted as $\gambleson$.
Given any subset $\os$ of $\vs$, we will define the \emph{linear hull} $\Span\group{\os}\coloneqq\cset{\sum_{k=1}^n\lambda_k\vect[u_k]}{n\in\nats,\lambda_k\in\reals,\vect[u_k]\in\os}\subseteq\vs$ and the \emph{positive hull} $\Posi\group{\os}\coloneqq\cset{\sum_{k=1}^n\lambda_k\vect[u_k]}{n\in\nats,\lambda_k\in\posreals,\vect[u_k]\in\os}\subseteq\Span\group{\os}$, where $\posreals$ is the set of all (strictly) positive real numbers.
A subset $\os$ of $\vs$ is called a \emph{convex cone} if it is closed under positive finite linear combinations, i.e. if $\Posi\group{\os}=\os$.
A convex cone $\mathcal{K}$ is called \emph{proper} if $\mathcal{K}\cap-\mathcal{K}=\set{0}$.
With any proper convex cone $\mathcal{K}\subseteq\vs$, we associate an ordering $\vo[\mathcal{K}]$ on $\vs$ as follows: $\vect\vo[\mathcal{K}]\vect[v]\iff\vect[v]-\vect\in\mathcal{K}$ for any $\vect$ and $\vect[v]$ in $\vs$.
For any $\vect$ and $\vect[v]$ in $\vs$, we write $\vect\svo[\mathcal{K}]\vect[v]$ if $\vect\vo[\mathcal{K}]\vect[v]$ and $\vect\neq\vect[v]$.
We collect all the options $\vect$ for which $0\svo[\mathcal{K}]\vect$ in $\vspos$.
When we work with gambles, then $\vs=\gambleson$ and the ordering will be the standard one $\leq$, given by $\g\leq\g[g]\iff\group{\forall\rval\in\ps}\g\group{\rval}\leq\g[g]\group{\rval}$.
We collect the positive gambles---gambles $\g$ for which $0<\g$---in $\posgambleson$.
Then $\leq$ corresponds to $\vo[\mathcal{K}]$ where we let $\mathcal{K}\coloneqq\posgambleson\cup\set{0}$.

We denote by $\cfdom\group{\vs}$ the set of all non-empty \emph{finite} subsets of $\vs$.
Elements of $\cfdom\group{\vs}$ are the option sets amongst which a subject can choose his preferred options.

A choice function $\cf$ on $\vs$ is a map $\cf\colon\cfdom\to\cfdom\cup\set{\emptyset}\colon\os\mapsto\cf\group{\os}$ such that $\cf\group{\os}\subseteq\os$.
Not every such map represents a rational belief; only the coherent choice functions do.
We call a choice function $\cf$ on $\vs$ \emph{coherent}\footnote{Our rationality axioms are based on those in~\cite{seidenfeld2010}, slightly modified for use with sets of desirable gambles.} if, for all $\os$, $\os[1]$ and $\os[2]$ in $\cfdom\group{\vs}$, all $\vect$ and $\vect[v]$ in $\vs$, and all $\lambda$ in $\posreals$:
\begin{enumerate}[noitemsep,topsep=0pt,label=\upshape C$_{\arabic*}$.,ref=\upshape C$_{\arabic*}$,leftmargin=*]
\item\label{coh cf 1: irreflexivity} $\cf(\os)\neq\emptyset$;
\item\label{coh cf 2: non-triviality} if $\vect\svo\vect[v]$ then $\set{\vect[v]}=\cf\group{\set{\vect,\vect[v]}}$;
\item\label{coh cf 3}
\begin{enumerate}[noitemsep,leftmargin=*,label=\upshape\alph*.,ref=\theenumi\upshape\alph*]
\item\label{coh cf 3a: alpha} if~$\cf(\os[2])\subseteq\os[2]\setminus\os[1]$ and $\os[1]\subseteq\os[2]\subseteq\os$ then~$\cf(\os)\subseteq\os\setminus\os[1]$;
\item\label{coh cf 3b: aizerman} if~$\cf(\os[2])\subseteq\os[1]$ and $\os\subseteq\os[2]\setminus\os[1]$ then~$\cf\group{\os[2]\setminus\os}\subseteq\os[1]$;
\end{enumerate}
\item\label{coh cf 4}
\begin{enumerate}[noitemsep,leftmargin=*,label=\upshape\alph*.,ref=\theenumi\upshape\alph*]
\item\label{coh cf 4a: scaling}
if~$\os[1]\subseteq\cf(\os[2])$ then~$\lambda\os[1]\subseteq\cf\group{\lambda\os[2]}$;
\item\label{coh cf 4b: independence}
if~$\os[1]\subseteq\cf(\os[2])$ then~$\os[1]+\set{\vect}\subseteq\cf\group{\os[2]+\set{\vect}}$.
\end{enumerate}
\end{enumerate}
Consider two isomorphic vector spaces $\vs$ and $\vstoo$, a linear order isomorphism $\bij$ between $\vs$ and $\vstoo$, and a choice function $\cf$ on $\vs$.
Define the choice function $\cf'$ on $\vstoo$ as $\vect\in\cf\group{\os}\iff\bij\group{\vect}\in\cf'\group{\bij\group{\os}}$ for all $\os$ in $\cfdom\group{\vs}$ and $\vect$ in $\os$.
Then, because $\bij$ is a bijection, $\cf$ satisfies Axioms~\ref{coh cf 1: irreflexivity} and~\ref{coh cf 3} if and only if $\cf'$ does; furthermore, because $\bij$ is order preserving, $\cf$ satisfies Axiom~\ref{coh cf 2: non-triviality} if and only if $\cf'$ does; and finally, because $\bij$ is linear, $\cf$ satisfies Axiom~\ref{coh cf 4} if and only if $\cf'$ does: such isomorphisms preserve coherence.

A set of desirable options (or gambles) $\sodv\subseteq\vs$ is essentially the restriction to pairwise comparison of a choice function: $\sodv=\cset{\vect\in\vs\setminus\set{0}}{\set{\vect}=\cf\group{\set{0,\vect}}}$.
We call $\sodv$ coherent if $0\notin\sodv$, $\vspos\subseteq\sodv$, $\vect\in\sodv\then\lambda\vect\in\sodv$, and $\vect,\vect[v]\in\sodv\then\vect+\vect[v]\in\sodv$ for all $\vect$ and $\vect[v]$ in $\vs$ and $\lambda$ in $\posreals$.
$\sodv$ is coherent if the choice function $\cf$ it is based on, is coherent.

Since, as we will see, an exchangeability assessment amounts to a specific indifference assessment, we recall how to model an indifference assessment.
For more information, we refer to~\cite[Sec.~5]{Vancamp2017}.
Next to $\cf\group{\os}$---the options that the agent \emph{strictly} prefers from $\os$---or $\sodv$---the options that he \emph{strictly} prefers to $0$---we consider the options that the agent considers to be \emph{equivalent to the zero option} $\soiv\subseteq\vs$.
We call $\soiv$ coherent if, for all $\vect$ and $\vect[v]$ in $\vs$ and $\lambda$ in $\reals$:
\begin{enumerate}[noitemsep,topsep=0pt,label=\upshape I$_{\arabic*}$.,ref=\upshape I$_{\arabic*}$,leftmargin=*]
\item\label{coh soiv 1: 0 is indifferent}
  $0\in\soiv$;
\item\label{coh soiv 2: positive or negative vectors are not indifferent}
  if $\vect\in\vspos\cup\vsneg$ then $\vect\notin\soiv$;
\item\label{coh soiv 3: scaling is indifferent}
  if $\vect\in\soiv$ then $\lambda\vect\in\soiv$;
\item\label{coh soiv 4: sum is indifferent}
  if $\vect,\vect[v]\in\soiv$ then $\vect+\vect[v]\in\soiv$.
\end{enumerate}
We collect all options that are indifferent to an option $\vect$ in $\vs$ into the \emph{equivalence class} $\ec\coloneq\cset{\vect[v]\in\vs}{\vect[v]-\vect\in\soiv}=\set{\vect}+\soiv$.
The set of all these equivalence classes is the \emph{quotient space} $\qs\coloneq\cset{\ec}{\vect\in\vs}$, being a linear space itself.
We provide it with the natural ordering inherited from $\vs$: $\tilde{\vect}\vo\tilde{\vect[v]}\iff\group{\exists\vect\in\tilde{\vect},\vect[v]\in\tilde{\vect[v]}}\vect\vo\vect[v]$ for all $\tilde{\vect}$ and $\tilde{\vect[v]}$ in $\qs$.

Consider any coherent set of indifferent options $\soiv$.
A choice function $\cf$ is then called \emph{compatible} with $\soiv$ if there is some \emph{representing} choice function $\cf'$ on $\qs$ such that $\cf\group{\os}=\cset{\vect\in\os}{\ec\in\cf'\group{\osind}}$ for all $\os$ in $\cfdom\group{\vs}$.
In that case, $\cf'$ is uniquely determined by $\cf'\group{\osind}=\osind[{\cf\group{\os}}]$ for all $\os$ in $\cfdom\group{\vs}$, and, moreover, $\cf$ is coherent if and only if $\cf'$ is.
Equivalently, we find the following useful characterisation: $\cf$ is compatible with $\soiv$ if and only if $0\in\cf\group{\os}\iff\vect\in\cf\group{\os}$ for all $\vect$ in $\soiv$ and $\os\supseteq\set{0,\vect}$ in $\cfdom\group{\vs}$, which corresponds to the definition of indifference in Ref.~\cite{seidenfeld1988}.

For desirability, compatibility with a coherent set of indifferent options $\soiv$ is defined as follows.
We call a set of desirable gambles $\sodv$ \emph{compatible} with $\soiv$ if $\sodv+\soiv\subseteq\sodv$, and this is equivalent to $\sodv=\bigcup\sodv'$ where $\sodv'\subseteq\qs$ is the \emph{representing} set of desirable options.
In that case, $\sodv'$ is uniquely given by $\sodv'=\osind[\sodv]$, and, moreover, $\sodv$ is coherent if and only if $\sodv'$ is.

\section{Finite exchangeability}
\label{sec:finite exchangeability}
Consider $\len\in\nats$ uncertain variables $\rv[1]$, \dots, $\rv[\len]$ taking values in a non-empty set~$\ps$.
The possibility space of the uncertain sequence $(\rv[1],\dots,\rv[\len])$ is $\ps^{\len}$.

We denote by $\rval=(\rval[1],\dots,\rval[\len])$ an arbitrary element of $\ps^{\len}$.
For any $\len$ in~$\nats$ we call $\permuts[\len]$ the group of all permutations $\permut$ of the index set $\set{1,\dots,\len}$.
There are $\abs{\permuts[\len]}=\len!$ such permutations.
With any such permutation $\permut$, we associate a permutation of $\ps^{\len}$, also denoted by $\permut$, and defined by $\group{\permut\rval}_k\coloneqq\rval[{\permut(k)}]$ for every $k$ in $\set{1,\dots,\len}$, or in other words, $\permut(\rval[1],\dots,\rval[\len])=(\rval[{\permut(1)}],\dots,\rval[{\permut(\len)}])$.
Similarly, we lift $\permut$ to a permutation $\lift$ on $\gambleson[\ps^{\len}]$ by letting $\lift\g=\g\circ\permut$, so $\group{\lift\g}\group{\rval}\coloneqq\g\group{\permut\rval}$ for all $\rval$ in $\ps^{\len}$.
Observe that $\lift$ is a linear permutation of the vector space $\gambleson[\ps^{\len}]$ of all gambles on $\ps^{\len}$.

If a subject assesses that the sequence of variables $\rv$ in $\ps^{\len}$ is exchangeable, this means precisely that he is indifferent between any gamble $\g$ on $\ps^{\len}$ and its permuted variant $\lift\g$, for any $\permut$ in $\permuts[\len]$. 
This leads us to the following proposal for the corresponding set of indifferent gambles:
\begin{equation*}
\soiv[{\permuts[\len]}]
\coloneqq
\Span\cset{\g-\lift\g}{\g\in\gambleson[\ps^{\len}],\permut\in\permuts[\len]}.
\end{equation*}

\begin{definition}\label{def:exchangeable choice function}
A choice function $\cf$ on $\gambleson[\ps^{\len}]$ is called \emph{(finitely) exchangeable} if\/ it is compatible with $\soiv[{\permuts[\len]}]$.
Similarly, a set of desirable gambles $\sodv\subseteq\gambleson[\ps^{\len}]$ is called \emph{(finitely) exchangeable} if\/ it is compatible with $\soiv[{\permuts[\len]}]$.
\end{definition}

Of course, so far, we do not yet know whether this notion of exchangeability is well-defined: indeed, we do not know yet whether $\soiv[{\permuts[\len]}]$ is a \emph{coherent} set of indifferent gambles.
In the next section, we will show that this is indeed the case.

\subsection{Count vectors}
In this section, we will provide the tools necessary to prove that $\soiv[{\permuts[\len]}]$ is a coherent set of indifferent gambles.
In~\cite{cooman2006d} and~\cite{cooman2010}, all the maps we use here are defined.

The \emph{permutation invariant atoms} $\pia\coloneqq\cset{\permut\rval}{\rval\in\ps^{\len}}$, $\rval$ in $\ps^{\len}$ are the smallest permutation invariant subsets of $\ps^{\len}$.
We introduce the \emph{counting map} $\cm\colon\ps^{\len}\to\cvs[\len]\colon\rval\mapsto\cm\group{\rval}$ where $\cm\group{\rval}$ is called the \emph{count vector} of $\rval$.
It is the $\ps$-tuple with components $\cm_z\group{\rval}\coloneqq\abs{\cset{k\in\set{1,\dots,\len}}{x_k=z}}$ for all $z$ in $\ps$, so $\cm_z$ is the number of times that $z$ occurs in the sequence $\rval[1]$, \dots, $\rval[\len]$.
The range of $\cm$---the set $\cvs[\len]$---is called the set of possible count vectors and is given by
\begin{equation*}
\cvs[\len]
\coloneqq
\cset[\bigg]{\cts\in\nonnegints^\ps}{\sum_{\rval\in\ps}\cts_{\rval}=\len}.
\end{equation*}
Remark that applying any permutation to $\rval$ leaves its result under the counting map unchanged:
\begin{equation*}
\cm\group{\rval}=\cm\group{\permut\rval}
\text{ for al $\rval$ in $\ps^{\len}$ and $\permut$ in $\permuts[\len]$}.
\end{equation*}
For any $\rval$ in $\ps^{\len}$, if $\cts=\cm\group{\rval}$ then $\pia=\cset{y\in\ps^{\len}}{\cm\group{y}=\cts}$, so the permutation invariant atom $\pia$ is completely determined by the count vector $\cts$ of all its elements, and is therefore also denoted by $\pia[{\cm\group{\rval}}]=\pia[\cts]$.
Remark that $\cset{\pia[\cts]}{\cts\in\cvs[\len]}$ partitions $\ps^{\len}$ into disjoint parts with constant count vectors, and that $\abs{\pia[\cts]}=\binom{\len}{\cts}\coloneqq\frac{N!}{\prod_{z\in\ps}m_z!}$.

In order to extend the idea of the count vectors for use with gambles, let us define the \emph{set of all permutation invariant gambles} as $\pigambleson[\ps^{\len}]\coloneqq\cset{\g\in\gambleson[\ps^{\len}]}{\group{\forall\permut\in\permuts[\len]}\lift\g=\g}\subseteq\gambleson[\ps^{\len}]$, and a special transformation $\ex$ of the linear space $\gambleson[\ps^{\len}]$
\begin{equation}\label{eq:def inv}
\ex
\colon
\gambleson[\ps^{\len}]\to\gambleson[\ps^{\len}]
\colon
\g\mapsto\ex\group{\g}\coloneqq\frac{1}{N!}\sum_{\permut\in\permuts[\len]}\lift\g,
\end{equation}
which, as we will see, is closely linked with $\pigambleson[\ps^{\len}]$ (see~\cite{cooman2010,Vancamp2017}).

\begin{proposition}\label{prop:ex properties}
$\ex$ is a linear transformation of $\gambleson[\ps^{\len}]$, and
\begin{enumerate}[label=\upshape(\roman*),leftmargin=*,noitemsep,topsep=0pt]
\item\label{it:ex property:invar} $\ex\circ\lift=\ex=\lift\circ\ex$ for all $\permut$ in $\permuts$;
\item\label{it:ex property:proj} $\ex\circ\ex=\ex$;
\item\label{it:ex property:kernel} $\kernel\group{\ex}=\soiv[{\permuts[\len]}]$;
\item\label{it:ex property:range} $\rng\group{\ex}=\pigambleson[\ps^{\len}]$.
\end{enumerate}
\end{proposition}
\noindent
So we see that $\ex$ is a linear projection operator that maps any gamble to a permutation invariant counterpart.

As shown by~\citet{cooman2010}, the linear projection operator $\ex$ renders a gamble insensitive to permutation by replacing it with the uniform average of all its permutations.
As a result, it assumes the same value for all gambles that can be related to each other through some permutation:
\begin{equation*}
\ex\group{\g}=\ex\group{\g[g]}
\text{ if $\g=\lift\g[g]$ for some $\permut$ in $\permuts[\len]$, for all $\g$ and $\g[g]$ in $\gambleson[\ps^{\len}]$.}
\end{equation*}
Furthermore, for any $\g$ in $\gambleson[\ps^{\len}]$, its transformation $\ex\group{\g}$ is permutation invariant and therefore constant on the permutation invariant atoms $\pia[\cts]$: $\group{\ex\group{\g}}\group{\rval}=\group{\ex\group{\g}}\group{y}$ if $\pia=\pia[y]$, for all $\rval$ and $y$ in $\ps^{\len}$.
We can use the properties of $\ex$ to prove that $\soiv[{\permuts[\len]}]$ is suitable for the definition of exchangeability.

\begin{proposition}\label{prop:soiv finite exchangeability is coherent}
Consider any $\len$ in $\nats$.
Then $\soiv[{\permuts[\len]}]$ is a coherent set of indifferent gambles.
\end{proposition}

Since $\soiv[{\permuts[\len]}]$ is coherent, exchangeability is well-defined, and by the discussion in Sec.~\ref{sec:choice_indiff}, the representing choice function $\cf'$ is defined on $\qsg$, and, similarly, the representing set of desirable gambles $\sodv'\subseteq\qsg$.
So we shall focus on the quotient space and its elements, exchangeable equivalent classes of gambles.

But before we do that, it will pay to further explore the notions we have introduced thus far.

Consider any $\g$ in $\gambleson[\ps^{\len}]$.
What is the constant value that $\ex\group{\g}$ assumes on a permutation invariant atom $\pia[\cts]$?
To answer this question, consider any $\rval$ in $\pia[\cts]$, then $\group{\ex\group{\g}}\group{\rval}=\frac{1}{n!}\sum_{\permut\in\permuts[\len]}\g\group{\permut\rval}=\frac{1}{n!}\frac{\abs{\permuts[\len]}}{\abs{\pia[\cts]}}\sum_{y\in\cset{\permut\rval}{\permut\in\permuts[\len]}}\g\group{y}=\frac{1}{\binom{\len}{\cts}}\sum_{y\in\pia}\g\group{y}=\frac{1}{\binom{\len}{\cts}}\sum_{y\in\pia[\cts]}\g\group{y}$,
whence 
\begin{equation}\label{eq:relation ex and Hy}
\ex=\sum_{\cts\in\cvs[\len]}\Hy\ind[{\pia[\cts]}],
\end{equation}
where $\Hy$ is the linear expectation operator associated with the uniform distribution on the invariant atom $\pia[\cts]$:
\begin{equation}\label{eq:def Hy}
\Hy[\len][\g][\cts]
\coloneqq
\frac{1}{\binom{\len}{\cts}}\sum_{y\in\pia[\cts]}\g\group{y}
\text{ for all $\g$ in $\gambleson[\ps^{\len}]$ and $\cts$ in $\cvs[\len]$.}
\end{equation}
It characterises a (multivariate) hyper-geometric distribution (see Ref.~\cite{johnson1997}), associated with random sampling without replacement from an urn with $\len$ balls of types $\ps$, whose composition is characterised by the count vector $\cts$.

The result of applying a gamble $\g$ on $\ps^{\len}$ to the map
\begin{equation}\label{eq:def Hymap}
\Hymap[\len]
\colon
\gambleson[\ps^{\len}]\to\gambleson[{\cvs[\len]}]
\colon
\g\mapsto\Hymap[\len]\group{\g}\coloneqq\Hy[\len][\g][\cdot]
\end{equation}
is the gamble $\Hymap[\len]\group{\g}$ on $\cvs[\len]$ that assumes  the value $\frac{1}{\binom{\len}{\cts}}\sum_{y\in\pia[\cts]}\g\group{y}$ for every $\cts$ in $\cvs[\len]$.

\subsection{Exchangeable equivalent classes of gambles}
\label{subsec:exchangeable ec of gambles}
We already know that exchangeable choice functions are represented by choice functions on the quotient space $\qsg$, and similar for sets of desirable gambles.
In the quest for an elegant representation theorem, we thus need to focus on the quotient space $\qsg$ and its elements, which are exchangeable equivalent classes of gambles.

In this section we investigate how the representation of permutation invariant gambles helps us find a representation for exchangeable choice functions.
To that end, the representation will use equivalence classes $\ec[\g]\coloneqq\set{\g}+\soiv[{\permuts[\len]}]$ of gambles, for any $\g$ in $\gambleson[\ps^{\len}]$.
Recall that the quotient space $\qsg\coloneqq\cset{\ec[\g]}{\g\in\gambleson[\ps^{\len}]}$ is a linear space itself, with additive identity $\ec[0]=\soiv[{\permuts[\len]}]$, and therefore any element $\tg$ of $\qsg$ is invariant under addition of $\soiv[{\permuts[\len]}]$: $\tg+\soiv[{\permuts[\len]}]=\tg$.
Elements of $\qsg$ will be generically denoted as $\tg$ or $\tg[g]$.


\begin{proposition}\label{prop:characterisation qsg by Hy}
Consider any $\g$ and $\g[g]$ in $\gambleson[\ps^{\len}]$.
Then $\ec[\g]=\ec[{\g[g]}]$ if and only if\/ $\Hymap[\len]\group{\g}=\Hymap[\len]\group{\g[g]}$
\end{proposition}

Therefore, it makes sense to introduce the map $\Hymape[\len]$:
\begin{equation}\label{eq:def Hymap on qsg}
\Hymape[\len]
\colon\qsg\to\gambleson[{\cvs[\len]}]
\colon\tg\mapsto\Hymap[\len]\group{\g}\text{ for any $\g$ in $\tg$.}
\end{equation}
Then Proposition~\ref{prop:characterisation qsg by Hy} guarantees that elements of $\qsg$ are characterised using $\Hymape[\len]$, in the sense that $\tg=\cset{\g\in\gambleson[\ps^n]}{\Hymap[\len]\group{\g}=\Hymape[\len]\group{\tg}}$ for all $\tg$ in $\qsg$.

The map $\Hymape[\len]$ takes as input an equivalence class of gambles, and maps it to some representing gamble on the count vectors.
It will be useful later on to consider some converse map $\Hymapei[\len]$:
\begin{equation}\label{eq:Hymap inverse}
\Hymapei[\len]
\colon
\gambleson[{\cvs[\len]}]\to\qsg
\colon
\g\mapsto\ec[{\sum_{\cts\in\cvs[\len]}\g\group{\cts}\ind[{\pia[\cts]}]}].
\end{equation}

\begin{proposition}\label{prop:Hymap inverse}
The maps $\Hymape[\len]$ as defined in Eq.~\eqref{eq:def Hymap on qsg} and $\Hymapei[\len]$ as defined in Eq.~\eqref{eq:Hymap inverse} are each other's inverses.
\end{proposition}

The importance of Prop.~\ref{prop:Hymap inverse} lies in the fact that now, $\Hymape[\len]$ is a bijection between $\qsg$ and $\gambleson[{\cvs[{\len}]}]$, and therefore, exchangeable equivalence classes of gambles are in a one-to-one correspondence with gambles on count vectors.

\begin{center}
\footnotesize
\begin{tikzpicture}[scale=0.8]
\node (1) at (0,0) {$\gambleson[\ps^{\len}]$};
\node (2) at (4,0) {$\gambleson[{\cvs[{\len}]}]$};
\node (3) at (0,-1.5) {$\qsg$};
\draw[->] (1) -- node[above] {$\Hymap[\len]$} (2);
\draw[->] (1) -- node[left] {$\ec[\cdot]$} (3);
\draw[<->] (2) -- node[below] {$\Hymape[\len]$} (3);
\end{tikzpicture}
\end{center}
The commuting diagram shows the surjections $\ec[\cdot]\colon\gambleson[\ps^{\len}]\to\qsg\colon\g\mapsto\ec[\g]$ and $\Hymap[\len]$ (indicated with a single arrow), and the bijection $\Hymape[\len]$ (indicated with a double arrow).
Since the representing choice function $\cftoo$ is defined from $\cf$ through $\ec[\cdot]$---working point-wise on sets---this already suggests that $\cftoo$ can be transformed into a choice function on $\gambleson[{\cvs[{\len}]}]$.
To prove that they preserve coherence, there is only one missing link: the ordering between $\qsg$ and $\gambleson[{\cvs[{\len}]}]$ should be preserved.

Therefore, to define the ordering $\vo$ on $\qsg$, as usual, we let $\vo$ be inherited by the ordering~$\leq$ on $\gambleson[\ps^{\len}]$:
\begin{equation}\label{eq:vector ordering:lexicographic}
\tg\vo\tg[g]
\iff
\group{\exists\g\in\tg,\exists\g[g]\in\tg[g]}
\g\leq\g[g]
\end{equation}
for all $\tg$ and $\tg[g]$ in $\qsg$, turning $\qsg$ into an ordered linear space.
It turns out that this vector ordering on $\qsg$ can be represented elegantly using $\Hymape[\len]$:
\begin{proposition}\label{prop:vector ordering:Hymap}
Consider any $\tg$ and $\tg[g]$ in $\qsg$, then $\tg\vo\tg[g]$ if and only if\/ $\Hymape[\len]\group{\tg}\leq\Hymape[\len]\group{\tg[g]}$.
\end{proposition}

Props.~\ref{prop:Hymap inverse} and~\ref{prop:vector ordering:Hymap} imply that $\Hymap[\len]$ is a linear order isomorphism.

\subsection{A representation theorem}
Now that we have found a linear order isomorphism $\Hymape[\len]$ between $\qsg$ and $\gambleson[{\cvs[\len]}]$,  we are ready to represent coherent and exchangeable choice functions.

\begin{theorem}[Finite Representation]\label{theorem:finite exchangeability representation}
Consider any choice function $\cf$ on $\gambleson[\ps^{\len}]$.
Then $\cf$ is exchangeable if and only if there is a unique representing choice function $\rcf$ on $\gambleson[{\cvs[\len]}]$ such that
\begin{equation*}
\cf\group{\os}
=
\cset{\g\in\os}{\Hymap[\len]\group{\g}\in\rcf\group{\Hymap[\len]\group{\os}}}
\text{ for all $\os$ in $\cfdom\group{\gambleson[\ps^{\len}]}$.}
\end{equation*}
Furthermore, in that case, $\rcf$ is given by $\rcf\group{\Hymap[\len]\group{\os}}=\Hymap[\len]\group{\cf\group{\os}}$ for all $\os$ in $\cfdom\group{\gambleson[\ps^{\len}]}$.
Finally, $\cf$ is coherent if and only if\/ $\rcf$ is.

Similarly, consider any set of desirable gambles $\sodv\subseteq\gambleson[\ps^{\len}]$.
Then $\sodv$ is exchangeable if and only if there is a unique representing set of desirable gambles $\tilde{\sodv}\subseteq\gambleson[{\cvs[\len]}]$ such that $\sodv=\bigcup\Hymapei[\len]\group{\tilde{\sodv}}$.
Furthermore, in that case, $\tilde{\sodv}$ is given by $\tilde{\sodv}=\Hymap[\len]\group{\sodv}$.
Finally, $\sodv$ is coherent if and only if $\tilde{\sodv}$ is.
\end{theorem}

The number of occurrences of any outcome in a sequence $(\rval[1],\dots,\rval[n])$ is fixed by its count vector~$\cts$ in $\cvs[\len]$.
If we impose an exchangeability assessment on it, then we see, using Theorem~\ref{theorem:finite exchangeability representation}, that the joint model on $\ps^n$ is characterised by a model on $\gambleson[{\cvs[\len]}]$.
So an exchangeable choice function $\cf$ essentially represents preferences between urns with $\len$ balls of types $\ps$ with different compositions~$\cts$: the choice $\cf\group{\os}$ between the gambles in $\os$ is based upon the composition $\cts$.

\subsection{Finite representation in terms of polynomials}
In Sec.~\ref{sec:countable exchangeability}, we will prove a similar representation theorem for infinite sequences.
Since it no longer makes sense to \emph{count} in such sequences, we first need to find a equivalent representation theorem in terms of something that does not depend on counts.
More specifically, we need, for every $\len$ in $\nats$ another order-isomorphic linear space to $\qsg$, that allows for embedding: the linear space for $n_1<n_2$ must be a subspace of the one for $n_2$.

All the maps we use here have been introduced by~\citet{cooman2006d}.
Moreover, we use their idea and work with polynomials on the $\ps$-simplex $\simplex\coloneqq\cset{\theta\in\reals^\ps}{\theta\geq0,\sum_{\rval\in\ps}\theta_{\rval}=1}$.
We consider the special subset $\polson$ of $\gambleson[\simplex]$: $\polson$ are the \emph{polynomial gambles} $\pol$ on $\simplex$, which are those gambles that are the restriction to $\simplex$ of a multivariate polynomial $p$ on $\reals^\ps$, in the sense that $\pol\group{\theta}=p\group{\theta}$ for all $\theta$ in $\simplex$.
We call $p$ then a representation of $\pol$.
It will be useful to introduce a notation for polynomial gambles with fixed degree $\len$ in $\nats$: $\polsond$ is the collection of all polynomial gambles that have at least one representation whose degree is not bigger than $\len$.
Both $\polson$ and $\polsond$ are linear subspaces of $\gambleson[\simplex]$, and, as wanted, for $\len[1]<\len[2]$, $\polsond[\simplex][{\len[1]}]$ is a subspace of $\polsond[\simplex][{\len[2]}]$.

Some special polynomial gambles are the \emph{Bernstein gambles}:
\begin{definition}[Bernstein gambles]\label{def:bernstein gambles}
Consider any $\len$ in $\nats$ and any $\cts$ in $\cvs[\len]$.
Define the \emph{Bernstein basis polynomial} $\bern$ on $\reals^\ps$ as $\bern\group{\theta}\coloneqq\binom{\len}{\cts}\prod_{\rval\in\ps}\theta_{\rval}^{\cts_{\rval}}$ for all $\theta$ in $\reals^\ps$.
The restriction to $\simplex$ is called a \emph{Bernstein gamble}, which we also denote as $\bern$.
\end{definition}
As shown in~\cite{cooman2010,debock2014}, the set of all Bernstein gambles forms a basis of the linear space $\polsond$:

\begin{proposition}\label{prop:bernstein gambles basis}
Consider any $\len$ in $\nats$.
The set of Bernstein gambles $\cset{\bern}{\cts\in\cvs[\len]}$ constitutes a basis of the linear space $\polsond$.
\end{proposition}

As we have seen, we need linear order isomorphisms to preserve coherence.
So we wonder whether there is one between $\qsg$ and $\polsond$.
In Sec.~\ref{subsec:exchangeable ec of gambles} we have seen that there is one between $\qsg$ and $\gambleson[{\cvs[\len]}]$, namely $\Hymape[\len]$.
Therefore, it suffices to find one between $\gambleson[{\cvs[\len]}]$ and $\polsond$.
Consider the map
\begin{equation}\label{eq:def CoMnmap}
\CoMnmap[\len]
\colon
\gambleson[{\cvs[\len]}]\to\polsond
\colon
\gc\mapsto\sum_{\cts\in\cvs[\len]}\gc\group{\cts}\bern.
\end{equation}
Before we can establish that $\CoMnmap[\len]$ is a linear order isomorphism, we need to provide the linear space $\polsond$ with an order $\vo[{\bern[]}]^n$.
We use the proper cone $\set{0}\cup\Posi\group{\cset{\bern}{\cts\in\cvs[\len]}}$ to define the order $\vo[{\bern[]}]^n$:
\begin{equation*}
\pol[h_1]\vo[{\bern[]}]^n\pol[h_2]
\iff
\pol[h_2]-\pol[h_1]\in\set{0}\cup\Posi\group{\cset{\bern}{\cts\in\cvs[\len]}}
\text{ for all $\pol[h_1]$ and $\pol[h_2]$ in $\polsond$.}
\end{equation*}

The following proposition is shown in~\cite{cooman2010}.

\begin{proposition}\label{prop:cvs and polsond are order isomorphic}
Consider any $\len$ in $\nats$.
Then the map $\CoMnmap[\len]$ is a linear order isomorphism between the ordered linear spaces $\gambleson[{\cvs[\len]}]$ and $\polsond$.
\end{proposition}

The linear order isomorphism $\CoMnmap[\len]$ helps us to define a linear order isomorphism between the linear spaces $\gambleson[{\ps^{\len}}]$ and $\polsond$, a final tool needed for a representation theorem in terms of polynomial gambles.
Indeed, consider for the map $\Mnmap[\len]\coloneqq\CoMnmap[\len]\circ\Hymap[\len]$:
\begin{equation*}
\Mnmap[\len]
\colon
\gambleson[\ps^{\len}]\to\polsond
\colon
\g\mapsto\Mn,
\end{equation*}
where $\Mn\coloneqq\sum_{\cts\in\cvs[\len]}\sum_{y\in\pia[\cts]}\g\group{y}\prod_{\rval\in\ps}\theta_{\rval}^{\cts_{\rval}}$ is the expectation of $\g$ associated with the multinomial distribution whose parameters are $\len$ and $\theta$.
We introduce its version
\begin{equation}\label{eq:def Mne}
\Mnmape[\len]\coloneqq\CoMnmap[\len]\circ\Hymape[\len],
\end{equation}
mapping $\qsg$ to $\polsond$.
There is an immediate connection between $\Mnmap[\len]$ and $\Mnmape[\len]$: they are both compositions of two linear order isomorphisms, and are therefore linear order isomorphisms themselves.
Due to Prop.~\ref{prop:characterisation qsg by Hy}, considering any $\tg$ in $\qsg$, $\Mnmap[\len]$ is constant on $\tg$, and the value it takes on any element of $\tg$ is exactly $\Mnmape[\len]\group{\tg}$.

\begin{center}
\footnotesize
\begin{tikzpicture}[scale=0.8]
\node (1) at (0,0.5) {$\gambleson[\ps^{\len}]$};
\node (2) at (-3,1.7) {$\gambleson[{\cvs[{\len}]}]$};
\node (3) at (3,1.7) {$\polsond$};
\node (4) at (0,-1) {$\qsg$};
\draw[->] (1) -- node[above right] {$\Hymap[\len]$} (2);
\draw[->] (1) -- node[above left] {$\Mnmap[\len]$} (3);
\draw[->] (1) -- node[right] {$\ec[\cdot]$} (4);
\draw[<->] (2) -- node[above] {$\CoMnmap[\len]$} (3);
\draw[<->] (3) -- node[below right] {$\Mnmape[\len]$} (4);
\draw[<->] (4) -- node[below left] {$\Hymape[\len]$} (2);
\end{tikzpicture}
\end{center}
The commuting diagram shows the surjections $\ec[\cdot]$, $\Hymap[\len]$ and $\Mnmap[\len]$, and the bijections $\Hymape[\len]$, $\Mnmape[\len]$ and $\CoMnmap[\len]$.
It shows that both $\gambleson[{\cvs[\len]}]$ and $\polsond$ are order-isomorphic to $\qsg$, so they are both suitable to define a representing choice function on.
In Theorem~\ref{theorem:finite exchangeability representation}, we used the space $\gambleson[{\cvs[\len]}]$.
Here, we will use the other equivalent space $\polsond$.

\begin{theorem}[Finite Representation]\label{theorem:finite exchangeability representation polynomial}
Consider any choice function $\cf$ on $\gambleson[\ps^{\len}]$.
Then $\cf$ is exchangeable if and only if there is a unique representing choice function $\rcf$ on $\polsond$ such that
\begin{equation*}
\cf\group{\os}
=
\cset{\g\in\os}{\Mnmap[\len]\group{\g}\in\rcf\group{\Mnmap[\len]\group{\os}}}
\text{ for all $\os$ in $\cfdom\group{\gambleson[\ps^{\len}]}$.}
\end{equation*}
Furthermore, in that case, $\rcf$ is given by $\rcf\group{\Mnmap[\len]\group{\os}}=\Mnmap[\len]\group{\cf\group{\os}}$ for all $\os$ in $\cfdom\group{\gambleson[\ps^{\len}]}$.
Finally, $\cf$ is coherent if and only if\/ $\rcf$ is.

Similarly, consider any set of desirable gambles $\sodv\subseteq\gambleson[\ps^{\len}]$.
Then $\sodv$ is exchangeable if and only if there is a unique representing set of desirable gambles $\tilde{\sodv}\subseteq\polsond$ such that $\sodv=\bigcup\Mnmapei[\len]\group{\tilde{\sodv}}$.
Furthermore, in that case, $\tilde{\sodv}$ is given by $\tilde{\sodv}=\Mnmap[\len]\group{\sodv}$.
Finally, $\sodv$ is coherent if and only if $\tilde{\sodv}$ is.
\end{theorem}

\section{Countable exchangeability}
\label{sec:countable exchangeability}
In the previous section, we assumed a finite sequence $\rv[1]$,\dots, $\rv[\len]$ to be exchangeable, and inferred representation theorems.
In this section, we will consider the whole sequence $\rv[1]$, \dots, $\rv[\len]$, \dots to be exchangeable, and derive representation theorems for such assessments.
We will call $\psall\coloneqq\bigtimes_{j\in\nats}\ps$, the set of all possible countable sequences where each variable takes values in $\ps$.

First, we will need a way to relate gambles on different domains.
Let $\g$ be some gamble on $\ps^n$, and let $\cg$ be its \emph{cylindrical extension}, defined as
\begin{equation*}
\cg\group{\rval[1],\dots,\rval[n],\dots}
\coloneqq
\g\group{\rval[1],\dots,\rval[n]}
\text{ for all $(\rval[1],\dots,\rval[n],\dots)$ in $\psall$.}
\end{equation*}
Formally, $\cg$ belongs to $\gambleson[\psall]$ while $\g$ belongs to $\gambleson[\ps^n]$.
However, they contain the same information, and therefore, are indistinguishable from a behavioural point of view.
In this paper, we will identify $\g$ with its cylindrical extension $\cg$.
Using this convention, we can for instance identify $\gambleson[\ps^{\len}]$ with a subset of $\gambleson[\psall]$, and, as an other example, for any $\mathcal{A}\subseteq\gambleson[\psall]$, regard $\mathcal{A}\cap\gambleson[\ps^{\len}]$ as those gambles in $\mathcal{A}$ that depend upon the first $\len$ variables only.

\subsection{Marginalisation}
Using the notational convention we just discussed, we can very easily define what marginalisation means for choice functions.
Given any choice function $\cf$ on $\gambleson[\psall]$ and any $\len$ in $\nats$, its $\ps^{\len}$-marginal $\cf_{\len}$ is determined by $\cf_{\len}\group{\os}\coloneqq\cf\group{\os}$ for all $\os$ in $\cfdom\group{\gambleson[\ps^{\len}]}$.

Similarly, given any set of desirable gambles $\sodv\subseteq\gambleson[\psall]$ and any $\len$ in $\nats$, its $\ps^{\len}$-marginal $\sodv[\len]$ is defined by $\sodv[\len]\coloneqq\sodv\cap\gambleson[\ps^{\len}]$.

Coherence is preserved under marginalisation [it is an immediate consequence of the definition; see, amongst others,~\cite[Proposition~6]{cooman2011b} for sets of desirable gambles]

\begin{proposition}\label{prop:marginalisation preserves coherence}
Consider any coherent choice function $\cf$ on $\gambleson[\psall]$ and any coherent set of desirable gambles $\sodv\subseteq\gambleson[\psall]$.
Then for every $\len$ in $\nats$, their $\ps^{\len}$-marginals $\cf_{\len}$ and $\sodv[\len]$ are coherent.
\end{proposition}

\subsection{Gambles of finite structure}
Before we can explain what it means to assess a countable sequence to be exchangeable, we need to realise that now there are (countably) infinite many variables.
However, we do not regard it useful from a behavioural point of view to choose between gambles that depend upon an infinite number of variables.
Indeed, since we will never be able to know the actual outcome, gambles will never be actually paid-off, and hence every assessment is essentially without any risk.
Instead, we believe that it makes sense to only consider choices between gambles of \emph{finite structure}: gambles that depend upon a finite number of variables only.
See~\cite{debock2014} for more information.

\begin{definition}[Gambles of finite structure]\label{def:Gambles of finite structure}
We will call any gamble that depends only upon a finite number of variables a \emph{gamble of finite structure}.
We collect all such gambles in $\gamblesfs$:
\begin{equation*}
\gamblesfs
\coloneqq
\cset{\g\in\gambleson[\psall]}{\group{\exists n\in\nats}\g\in\gambleson[\ps^{\len}]}
=
\bigcup_{n\in\nats}\gambleson[\ps^{\len}].
\end{equation*}
\end{definition}
$\gamblesfs$ is a linear space, with the usual ordering~$\leq$: for any $\g$ and $\g[g]$ in $\gamblesfs$, $\g\leq\g[g]\iff\g\group{\rval}\leq\g[g]\group{\rval}$ for all $\rval$ in $\psall$.

Due to our finitary context, we can even establish a converse result to Prop.~\ref{prop:marginalisation preserves coherence}, whose proof for the part about sets of desirable gambles can be found in~\cite[Proposition~4]{debock2014}, and for the part about choice functions is omitted since it is a straight-forward check of all the axioms.

\begin{proposition}\label{prop:coherence marginalisation finitary}
Consider any choice function $\cf$ on $\gamblesfs$, and any set of desirable gambles $\sodv\subseteq\gamblesfs$.
If for every $\len$ in $\nats$, its $\ps^{\len}$-marginal $\cf_{\len}$ on $\gambleson[\ps^{\len}]$ is coherent, then $\cf$ is coherent.
Similarly, if for every $\len$ in $\nats$, its $\ps^{\len}$-marginal $\sodv[\len]\subseteq\gambleson[\ps^{\len}]$ is coherent, then $\sodv$ is coherent.
\end{proposition}

\subsection{Set of indifferent gambles}

If a subject assesses the sequence of variables $\rv[1]$, \dots, $\rv[\len]$, \dots to be exchangeable, this means that he is indifferent between any gamble $\g$ in $\gamblesfs$ and its permuted variant $\lift\g$, for any $\permut$ in $\permuts[\len]$, where $\len$ now is the (finite) number of variables that $\g$ depends upon: his set of indifferent gambles is
\begin{equation*}
\soiv[\permuts]
\coloneqq
\cset{\g\in\gamblesfs}{\group{\exists\len\in\nats}\g\in\soiv[{\permuts[\len]}]}
=
\bigcup_{\len\in\nats}\soiv[{\permuts[\len]}].
\end{equation*}
If we want to use $\soiv[\permuts]$ to define countable exchangeability, it must be a coherent set of indifferent gambles.

\begin{proposition}\label{prop:soiv countable exchangeability is coherent}
The set $\soiv[\permuts]$ is a coherent set of indifferent gambles.
\end{proposition}


Countable exchangeability is now easily defined, similar to the definition for the finite case.

\begin{definition}\label{def:exchangeable choice function countable}
A choice function $\cf$ on $\gamblesfs$ is called \emph{(countably) exchangeable} if\/ $\cf$ is compatible with $\soiv[\permuts]$.
Similarly, a set of desirable gambles $\sodv\subseteq\gamblesfs$ is called \emph{(countably) exchangeable} if it is compatible with $\soiv[\permuts]$.
\end{definition}

This definition is closely related to its finite counterpart.

\begin{proposition}\label{prop:connection countable exchangeability and finite exchangeability}
Consider any coherent choice function $\cf$ on $\gamblesfs$.
Then $\cf$ is exchangeable if and only if for every choice of $\len$ in $\nats$, the $\ps^{\len}$-marginal $\cf_{\len}$ of $\cf$ is exchangeable.
Similarly, consider any coherent set of desirable gambles $\sodv\subseteq\gamblesfs$.
Then $\sodv$ is exchangeable if and only if for every choice of $\len$ in $\nats$, the $\ps^{\len}$-marginal $\sodv[\len]$ of $\sodv$ is exchangeable.
\end{proposition}

\subsection{A representation theorem for countable sequences}
We will look for a similar representation result.
However, since we no longer deal with finite sequences of length $\len$, now the representing choice function won't be defined on $\polsond$, but instead on $\polson$.

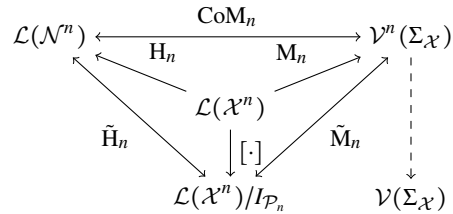
\begin{center}
\footnotesize
\begin{tikzpicture}[scale=0.8]
\node (1) at (0,0.5) {$\gambleson[\ps^{\len}]$};
\node (2) at (-3,1.7) {$\gambleson[{\cvs[{\len}]}]$};
\node (3) at (3,1.7) {$\polsond$};
\node (4) at (0,-1) {$\qsg$};
\node (5) at (3,-1) {$\polson$};
\draw[->] (1) -- node[above right] {$\Hymap[\len]$} (2);
\draw[->] (1) -- node[above left] {$\Mnmap[\len]$} (3);
\draw[->] (1) -- node[right] {$\ec[\cdot]$} (4);
\draw[<->] (2) -- node[above] {$\CoMnmap[\len]$} (3);
\draw[<->] (3) -- node[below right] {$\Mnmape[\len]$} (4);
\draw[<->] (4) -- node[below left] {$\Hymape[\len]$} (2);
\draw[dashed,->] (3) -- (5);
\end{tikzpicture}
\end{center}
In the commuting diagram, a dashed line represents an embedding: indeed, for every $\len$ in $\nats$, $\polsond$ is a subspace of $\polson$.
That shows the importance of the polynomial representation.

As we have seen, in order to define coherent choice functions on some linear space, we need to provide it with a vector ordering.
Similar to what we did before, we use the proper cone $\set{0}\cup\Posi\group{\cset{\bern}{\cts\in\cvs[\len],n\in\nats}}$ to define the order $\vo[{\bern[]}]$ on $\polson$:
\begin{equation*}
\pol[h_1]\vo[{\bern[]}]\pol[h_2]
\iff
\pol[h_2]-\pol[h_1]\in\set{0}\cup\Posi\group{\cset{\bern}{\cts\in\cvs[\len],n\in\nats}}
\end{equation*}
for all $\pol[h_1]$ and $\pol[h_2]$ in $\polson$.

Keeping Props.~\ref{prop:marginalisation preserves coherence} and~\ref{prop:coherence marginalisation finitary} in mind, the following result is not surprising.

\begin{proposition}\label{prop:bernstein coherence choice functions}
Consider any choice function $\cftoo$ on $\polson$.
Then $\cftoo$ is coherent if and only if for every $\len$ in $\nats$ the choice function $\cftoo_n$, given by $\cftoo_n\group{\os}\coloneqq\cftoo\group{\os}$ for all $\os$ in $\cfdom\group{\polsond}$ is coherent.
\end{proposition}

\begin{theorem}[Countable Representation]\label{theorem:countable exchangeability representation}
Consider any choice function $\cf$ on $\gamblesfs$.
Then $\cf$ is exchangeable if and only if there is a unique representing choice function $\rcf$ on $\polson$ such that, for every $\len$ in $\nats$, the $\ps^{\len}$-marginal $\cf_{\len}$ of $\cf$ is determined by
\begin{equation*}
\cf_{\len}\group{\os}
=
\cset{\g\in\os}{\Mnmap[\len]\group{\g}\in\rcf\group{\Mnmap[\len]\group{\os}}}
\text{ for all $\os$ in $\cfdom\group{\gambleson[\ps^{\len}]}$.}
\end{equation*}
Furthermore, in that case, $\rcf$ is given by $\rcf\group{\os}\coloneqq\bigcup_{\len\in\nats}\rcf_{\len}\group{\os\cap\polsond}$ for all $\os$ in $\cfdom\group{\polson}$,
with $\rcf_{\len}\group{\Mnmap[\len]\group{\os}}\coloneqq\Mnmap[\len]\group{\cf_{\len}\group{\os}}$ for every $\os$ in $\cfdom\group{\gambleson[\ps^{\len}]}$, and where we let $\rcf_{\len}\group{\emptyset}\coloneqq\emptyset$ for notational convenience.
Finally, $\cf$ is coherent if and only if\/ $\rcf$ is.

Similarly, consider any set of desirable gambles $\sodv\subseteq\gamblesfs$.
Then $\sodv$ is exchangeable if and only if there is a unique representing $\tilde{\sodv}\subseteq\polson$ such that, for every $\len$ in $\nats$, the $\ps^{\len}$-marginal $\sodv[\len]$ is given by $\sodv[\len]=\bigcup\Mnmapei[\len]\group{\tilde{\sodv}\cap\polsond}$.
Furthermore, in that case, $\tilde{\sodv}$ is given by $\tilde{\sodv}=\bigcup_{\len\in\nats}\Mnmap[\len]\group{\sodv[\len]}$.
Finally, $\sodv$ is coherent if and only if $\tilde{\sodv}$ is.
\end{theorem}

\section{Conclusion}
We studied exchangeability and we have found counterparts to de Finetti's finite and countable representation results, in the general setting of choice functions.
We have shown that an exchangeability assessment is a particular indifference assessment, where we identified the set of indifferent options.
The main idea that made (finite) representation possible is the linear order isomorphism $\Hymapei[\len]$ between the quotient space and the set of gambles on count vectors, indicating that (finitely) exchangeable choice functions can be represented by a choice function that essentially represents preferences between urns with $\len$ balls of types $\ps$ with different compositions $\cts$.
Alternatively, for the countable case, we have shown that there is a polynomial representation.

Choice functions form a belief structure (see~\cite{Vancamp2017}).
Therefore, any infimum of coherent choice functions is a coherent choice function itself.
Since any infimum of choice functions compatible with some fixed set of indifferent options $\soiv$, is compatible with $\soiv$ as well (see~\cite{Vancamp2017}), our results indicate that, using choice functions, it is conceptually easy to reason about exchangeable sequences: infima of exchangeable and coherent choice functions will be exchangeable and coherent as well.

A possible future goal is to investigate how exchangeability behaves under updating.
It is shown, in~\cite{cooman2010}, that, for exchangeable sets of desirable gambles, updating can be done directly for the representing set of desirable gambles in the count space.
We expect this to be the case for choice functions as well.

\section*{Acknowledgments}
Gert de Cooman's research was partly funded through project number 3G012512 of the Research Foundation Flanders (FWO).
The authors would like to thank Enrique Miranda for his comments on a draft of the paper.

\appendix\newpage
\section{Proofs of some results}


\begin{proof}[Proof of Prop.~\ref{prop:soiv finite exchangeability is coherent}]
For Axiom~\ref{coh soiv 1: 0 is indifferent}, since $\soiv[{\permuts[\len]}]$ is a linear span, $\vect[0]$ is included in $\soiv[{\permuts[\len]}]$.
For Axiom~\ref{coh soiv 2: positive or negative vectors are not indifferent}, consider any $\g$ in $\soiv[{\permuts[\len]}]$ and assume \emph{ex absurdo} that $\g\in\posgambleson[\ps^{\len}]\cup\neggambleson[\ps^{\len}]$.
If $\g\in\posgambleson[\ps^{\len}]$ then $\lift\g\in\posgambleson[\ps^{\len}]$ for all $\permut$ in $\permuts[\len]$, and therefore $\ex\group{\g}>0$, a contradiction with Prop.~\ref{prop:ex properties}\ref{it:ex property:kernel}.
If $\g\in\neggambleson[\ps^{\len}]$ then, similarly $\ex\group{\g}<0$, again a contradiction with Prop.~\ref{prop:ex properties}\ref{it:ex property:kernel}.
Axioms~\ref{coh soiv 3: scaling is indifferent} and~\ref{coh soiv 4: sum is indifferent} are satisfied because $\soiv[{\permuts[\len]}]$ is a linear span.
\end{proof}

\begin{proof}[Proof of Prop.~\ref{prop:characterisation qsg by Hy}]
Infer the following equivalences.
Start with $\ec[\g]=\ec[{\g[g]}]$, what, due to the definition of equivalence classes, is equivalent to $\g+\g[h_1]=\g[g]+\g[h_2]$, and equivalently, $\g[g]-\g=\g[h_1]-\g[h_2]$, for some $\g[h_1]$ and $\g[h_2]$ in $\soiv[{\permuts[\len]}]$.
In turn, since $\soiv[{\permuts[\len]}]$ is a linear space, that is equivalent to $\g[g]-\g\in\soiv[{\permuts[\len]}]$.
Because $\ker\group{\ex}=\soiv[{\permuts[\len]}]$ [by Proposition~\ref{prop:ex properties}\ref{it:ex property:kernel}], we find equivalently that $\ex\group{\g[g]-\g}=0$, and, due to the linearity of $\ex$, equivalently $\ex\group{\g}=\ex\group{\g[g]}$.
Use Lemma~\ref{lemma:relation ex and Hymap} to find that, indeed, equivalently $\Hymap[\len]\group{\g}=\Hymap[\len]\group{\g[g]}$.
\end{proof}

\begin{lemma}\label{lemma:relation ex and Hymap}
Consider any $\g$ and $\g[g]$ in $\gambleson[\ps^{\len}]$.
Then $\ex\group{\g}=\ex\group{\g[g]}$ if and only if\/ $\Hymap[\len]\group{\g}=\Hymap[\len]\group{\g[g]}$.
\end{lemma}

\begin{proof}[Proof of Lemma~\ref{lemma:relation ex and Hymap}]
Infer the following equivalences.
Start with $\ex\group{\g}=\ex\group{\g[g]}$, what, by Eq.~\eqref{eq:relation ex and Hy}, is equivalent to $\sum_{\cts\in\cvs[\len]}\Hy[\len][\g][\cts]\ind[{\pia[\cts]}]=\sum_{\cts\in\cvs[\len]}\Hy[\len][{\g[g]}][\cts]\ind[{\pia[\cts]}]$.
Equivalently, we find that $\Hy[\len][\g][\cts]=\Hy[\len][{\g[g]}][\cts]$ for all $\cts$ in $\cvs[\len]$.
In turn, by Eq.~\eqref{eq:def Hymap}, that is equivalent to $\Hymap[\len]\group{\g}=\Hymap[\len]\group{\g[g]}$.
\end{proof}

\begin{proof}[Proof of Prop.~\ref{prop:Hymap inverse}]
This proof is structured as follows: we show that (i) $\Hymapei[\len]\circ\Hymape[\len]=\id[\qsg]$, and (ii) $\Hymape[\len]\circ\Hymapei[\len]=\id[{\gambleson[{\cvs[\len]}]}]$, together implying that $\Hymape[\len]$ and $\Hymapei[\len]$ are each other's inverses.

For (i), consider any $\tg$ in $\qsg$.
We need to show that then $\Hymapei[\len]\group{\Hymape[\len]\group{\tg}}=\tg$.
Let $\g[h]$ be an arbitrary element of $\tg$, and $\g\coloneqq\ex\group{\g[h]}$.
Then $\ex\group{\g}=\ex\group{\g[h]}$ by Prop.~\ref{prop:ex properties}\ref{it:ex property:proj}, and therefore, using Lemma~\ref{lemma:relation ex and Hymap}, $\Hymap[\len]\group{\g}=\Hymap[\len]\group{\g[h]}$, so Prop.~\ref{prop:characterisation qsg by Hy} implies that $\g\in\tg$ as well.
Then $\Hymape[\len]\group{\tg}$ assumes the value $\Hymap[\len]\group{\g}\group{\cts}=\frac{1}{\binom{\len}{\cts}}\sum_{y\in\pia[\cts]}\g\group{y}$ on every $\cts$ in $\cvs[\len]$.
But $\g$ is constant on every permutation atom $\pia[\cts]$, so $\frac{1}{\binom{\len}{\cts}}\sum_{y\in\pia[\cts]}\g\group{y}=\frac{1}{\binom{\len}{\cts}}\abs{\pia[\cts]}\g\group{\rval}=\g\group{\rval}$ for every $\rval$ in $\pia[\cts]$, and therefore
\begin{equation}\label{eq:proof:Hymap inverse}
f
=
\sum_{\cts\in\cvs[\len]}\Hymap\group{\g}\group{\cts}\ind[{\pia[\cts]}]
=
\sum_{\cts\in\cvs[\len]}\Hymape\group{\tg}\group{\cts}\ind[{\pia[\cts]}].
\end{equation}
Then indeed $\Hymapei[\len]\group{\Hymape[\len]\group{\tg}}=\ec[{\sum_{\cts\in\cvs[\len]}\Hymape\group{\tg}\group{\cts}\ind[{\pia[\cts]}]}]=\ec[\g]=\tg$, where the first equality follows from Eq.~\eqref{eq:Hymap inverse}, the second one from Eq.~\eqref{eq:proof:Hymap inverse}, and the last one from the fact that $\g\in\tg$ and $\ec[\g]=\set{\g}+\soiv[{\permuts[\len]}]$.

For (ii), consider any $\g$ in $\gambleson[{\cvs[\len]}]$.
We need to show that then $\Hymape[\len]\group{\Hymapei[\len]\group{\g}}=\g$.
Let $\g[g]\coloneqq\sum_{\cts\in\cvs[\len]}\g\group{\cts}\ind[{\pia[\cts]}]$, being a gamble on $\ps^{\len}$.
Then $\Hymapei[\len]\group{\g}=\ec[{\g[g]}]$ by Eq.~\eqref{eq:Hymap inverse}, so $\Hymape[\len]\group{\Hymapei[\len]\group{\g}}=\Hymape[\len]\group{\ec[\g[g]]}$, and since $\g[g]\in\ec[\g[g]]$, we find using Eq.~\eqref{eq:def Hymap on qsg} that $\Hymape[\len]\group{\ec[\g[g]]}=\Hymap[\len]\group{\g[g]}$ and therefore $\Hymape[\len]\group{\Hymapei[\len]\group{\g}}=\Hymap[\len]\group{\g[g]}$.
The proof is finished if we show that $\Hymap[\len]\group{\g[g]}=\g$.
Consider any $\cts'$ in $\cvs[\len]$, and infer that
\begin{align*}
\Hymap[\len]\group{\g[g]}\group{\cts'}
=
\frac{1}{\binom{\len}{\cts'}}\sum_{y\in\pia[\cts']}\g[g]\group{y}
&=
\frac{1}{\binom{\len}{\cts'}}\sum_{y\in\pia[\cts']}\sum_{\cts\in\cvs[\len]}\g\group{\cts}\ind[{\pia[\cts]}]\group{y}\\
&=
\sum_{\cts\in\cvs[\len]}\g\group{\cts}\frac{1}{\binom{\len}{\cts'}}\sum_{y\in\pia[\cts']}\ind[{\pia[\cts]}]\group{y}
=
\sum_{\cts\in\cvs[\len]}\g\group{\cts}\ind[\set{\cts}]\group{\cts'}
=
\g\group{\cts'},
\end{align*}
where the first equality follows from Eq.~\eqref{eq:def Hy} and the penultimate from the fact that
\begin{equation*}
\frac{1}{\binom{\len}{\cts'}}\sum_{y\in\pia[\cts']}\ind[{\pia[\cts]}]\group{y}
=
\begin{cases}
1&\text{if $\cts'=\cts$,}\\
0&\text{otherwise.}
\end{cases}
\end{equation*}
Therefore indeed $\Hymape[\len]\group{\Hymapei[\len]\group{\g}}=\Hymap[\len]\group{\g[g]}=\g$.
\end{proof}

\begin{proof}[Proof of Prop.~\ref{prop:vector ordering:Hymap}]
For necessity, assume that $\tg\vo\tg[g]$.
Then, by Eq.~\eqref{eq:vector ordering:lexicographic}, $\g\leq\g[g]$ for some $\g$ in $\tg$ and $\g[g]$ in $\tg[g]$.
Consider any $\cts$ in $\cvs[\len]$, and infer that $\Hymap[\len]\group{\g}\group{\cts}=\frac{1}{\binom{\len}{\cts}}\sum_{y\in\pia[\cts]}\g\group{y}\leq\frac{1}{\binom{\len}{\cts}}\sum_{y\in\pia[\cts]}\g[g]\group{y}=\Hymap[\len]\group{\g[g]}\group{\cts}$.
Then $\Hymap[\len]\group{\g}\leq\Hymap[\len]\group{\g[g]}$, and therefore, by Eq.~\eqref{eq:def Hymap on qsg}, indeed $\Hymape[\len]\group{\tg}\leq\Hymape[\len]\group{\tg[g]}$.

For sufficiency, assume that $\Hymape[\len]\group{\tg}\leq\Hymape[\len]\group{\tg[g]}$.
Then, by Eq.~\eqref{eq:def Hymap on qsg} and Prop.~\ref{prop:characterisation qsg by Hy}, $\Hymap[\len]\group{\g}\leq\Hymap[\len]\group{\g[g]}$ for all $\g$ in $\tg$ and $\g[g]$ in $\tg[g]$.
Consider any $\g$ in $\tg$ and $\g[g]$ in $\tg[g]$ and let $\g'\coloneqq\ex\group{\g}$ and $\g[g]'\coloneqq\ex\group{\g[g]}$.
Then $\ex\group{\g'}=\ex\group{\g}$ and $\ex\group{\g[g]'}=\ex\group{\g[g]}$ by Prop.~\ref{prop:ex properties}\ref{it:ex property:proj}, so Lemma~\ref{lemma:relation ex and Hymap} and Prop.~\ref{prop:characterisation qsg by Hy} together imply that $\g'\in\tg$ and $\g[g]'\in\tg$, so $\Hymap[\len]\group{\g'}\leq\Hymap[\len]\group{\g[g]'}$.
Then, by Eqs.~\eqref{eq:def Hymap} and~\eqref{eq:def Hy}, $\frac{1}{\binom{\len}{\cts}}\sum_{y\in\pia[\cts]}\g'\group{y}\leq\frac{1}{\binom{\len}{\cts}}\sum_{y\in\pia[\cts]}\g[g]'\group{y}$ for every $\cts$ in $\cvs[\len]$.
But $\g'$ and $\g[g]'$ are constant on every $\pia[\cts]$, so $\g'\group{y}\leq\g[g]'\group{y}$ for every $y$ in $\pia[\cts]$ and every $\cts$ in $\cvs[\len]$.
Then $\g'\leq\g[g]'$, and therefore indeed $\tg\vo\tg[g]$.
\end{proof}

\begin{proof}[Proof of Theorem~\ref{theorem:finite exchangeability representation}]
We begin with the representation of choice functions.
For the first statement, note that $\cf$ is exchangeability is equivalent to compatibility with $\soiv[{\permuts[\len]}]$ [by Def.~\ref{def:exchangeable choice function}], and equivalently, there is some representing choice function $\cftoo$ on $\qsg$ such that $\cf\group{\os}=\cset{\g\in\os}{\ec[\g]\in\cftoo\group{\osin}}$ for all $\os$ in $\cfdom\group{\gambleson[\ps^{\len}]}$.
We use the linear order isomorphism $\Hymape[\len]$ to define a choice function $\rcf$ on $\gambleson[{\cvs[\len]}]$: we let $\ec[\g]\in\cftoo\group{\osin}\iff\Hymape[\len]\group{\ec[\g]}\in\rcf\group{\Hymape[\len]\group{\osin}}$ for all 
$\g$ in $\gambleson[\ps^{\len}]$ and $\os$ in $\cfdom\group{\gambleson[\ps^{\len}]}$.
Since $\g\in\ec[\g]$, use Prop.~\ref{prop:characterisation qsg by Hy} and Eq.~\eqref{eq:def Hymap on qsg} to infer that $\Hymape[\len]\group{\ec[\g]}=\Hymap[\len]\group{\g}$.
Similarly, infer that $\Hymape[\len]\group{\osin}=\cset{\Hymape[\len]\group{\ec[\g[g]]}}{\g[g]\in\os}=\cset{\Hymap[\len]\group{\g[g]}}{\g[g]\in\os}=\Hymap[\len]\group{\os}$, so $\ec[\g]\in\cftoo\group{\osin}\iff\Hymap[\len]\group{\g}\in\rcf\group{\Hymap[\len]\group{\os}}$.
Then indeed
\begin{equation*}
\cf\group{\os}
=
\cset{\g\in\os}{\Hymap[\len]\group{\g}\in\rcf\group{\Hymap[\len]\group{\os}}}
\text{ for all $\os$ in $\cfdom\group{\gambleson[\ps^{\len}]}$.}
\end{equation*}
To show that $\rcf$ is unique, use that $\cftoo$ is unique and $\Hymape[\len]$ is a bijection to infer that $\rcf$ is unique too.

For the second statement, consider any $\os$ in $\cfdom\group{\gambleson[{\ps^{\len}}]}$ and infer, using the definition of $\rcf$, that $\rcf\group{\Hymape[\len]\group{\osin}}=\Hymape[\len]\group{\cftoo\group{\osin}}$ , and therefore $\rcf\group{\Hymap[\len]\group{\os}}=\Hymape[\len]\group{\cftoo\group{\osin}}$.
Since $\cftoo$ is given by $\cftoo\group{\osin}=\osin[{\cf\group{\os}}]$, we find $\Hymape[\len]\group{\cftoo\group{\osin}}=\Hymape[\len]\group{\osin[{\cf\group{\os}}]}=\Hymap[\len]\group{\cf\group{\os}}$, and therefore indeed $\rcf\group{\Hymap[\len]\group{\os}}=\Hymap[\len]\group{\cf\group{\os}}$.

For the third statement, by the compatibility with $\soiv[{\permuts[\len]}]$ guarantees that $\cf$ is coherent if and only if $\cftoo$ on $\qsg$ is coherent.
But since $\rcf$ is defined from $\cftoo$ using the linear order isomorphism $\Hymap[\len]$, we have immediately that $\rcf$ is coherent if and only if $\cftoo$ is coherent.
Therefore indeed $\cf$ is coherent if and only if $\rcf$ is coherent.

We now turn to the representation for sets of desirable gambles.
Since $\sodv$ is compatible with $\soiv[{\permuts[\len]}]$, there is some representing set of desirable gambles $\sodv'\subseteq\qsg$ such that $\sodv=\bigcup\sodv'$.
Using the linear order isomorphism $\Hymape[\len]$ we can transform $\sodv'$ to $\tilde{\sodv}\coloneqq\Hymape[\len]\group{\sodv'}$ on the isomorphic space $\gambleson[{\cvs[\len]}]$.
Then $\sodv'=\Hymapei[\len]\group{\tilde{\sodv}}$, so indeed $\sodv=\bigcup\Hymapei[\len]\group{\tilde{\sodv}}$.
Because compatibility with $\soiv[{\permuts[\len]}]$ guarantees that $\sodv'$---and therefore also $\tilde{\sodv}$---is unique, we have shown the first statement to be true.

For the second statement, infer that indeed $\tilde{\sodv}=\Hymape[\len]\group{\sodv'}=\Hymap[\len]\group{\sodv}$ by Eq.~\eqref{eq:def Hymap on qsg}.

For the third statement, compatibility with $\soiv[{\permuts[\len]}]$ guarantees that $\sodv$ is coherent if and only if $\sodv'$ is coherent.
But since $\tilde{\sodv}$ is defined from $\sodv'$ using the linear order isomorphism $\Hymap[\len]$, we have immediately that $\tilde{\sodv}$ is coherent if and only if $\sodv'$ is coherent.
Therefore indeed $\sodv$ is coherent if and only if $\tilde{\sodv}$ is coherent.
\end{proof}

\begin{proof}[Proof of Theorem~\ref{theorem:finite exchangeability representation polynomial}]
Let $\cf''$ on $\gambleson[{\cvs[\len]}]$ and $\sodv''\subseteq\gambleson[{\cvs[\len]}]$ be the representing choice function and set of desirable gambles from Theorem~\ref{theorem:finite exchangeability representation}, and let $\rcf$ be defined by
\begin{equation*}
\CoMnmap[\len]\group{\g}\in\rcf\group{\CoMnmap[\len]\group{\os}}\iff\g\in\cf''\group{\os}
\end{equation*}
for all $\os$ in $\cfdom\group{\gambleson[{\cvs[\len]}]}$, and $\tilde{\sodv}\coloneqq\CoMnmap[\len]\group{\sodv''}$.
Since $\CoMnmap[\len]$ is a linear order isomorphism, then $\Mnmap[\len]\group{\g}\in\rcf\group{\Mnmap[\len]\group{\os}}\iff\Hymap[\len]\group{\g}\in\cf''\group{\Hymap[\len]\group{\os}}$ for all $\os$ in $\cfdom\group{\gambleson[\ps^{\len}]}$, and $\tilde{\sodv}=\CoMnmap[\len]\group{\Hymap[\len]\group{\sodv}}=\Mnmap[\len]\group{\sodv}$, and all the coherence properties are preserved, from which the statements follow.
\end{proof}

\begin{proof}[Proof of Prop.~\ref{prop:soiv countable exchangeability is coherent}]
For Axiom~\ref{coh soiv 1: 0 is indifferent}, since, by Prop.~\ref{prop:soiv finite exchangeability is coherent}, $0\in\soiv[{\permuts[\len]}]$ for every $\len$ in $\nats$, also $0\in\soiv[\permuts]$.
For Axiom~\ref{coh soiv 2: positive or negative vectors are not indifferent}, consider any $\g$ in $\soiv[\permuts]$, then there is some $\len$ in $\nats$ for which $\g\in\soiv[{\permuts[\len]}]$
By Prop.~\ref{prop:soiv finite exchangeability is coherent}, we infer that indeed $\g\not<0$ and $\g\not>0$.
For Axioms~\ref{coh soiv 3: scaling is indifferent} and~\ref{coh soiv 4: sum is indifferent}, consider any $\g[f_1]$, $\g[f_2]$ and $\g[f_3]$ in $\soiv[\permuts]$ and any $\lambda$ in $\reals$.
Then there are $n_i$ in $\nats$ such that $\g[f_i]\in\soiv[{\permuts[n_i]}]$, for every $i$ in $\set{1,2,3}$.
Let $n\coloneqq\max\set{n_1,n_2,n_3}$.
Then $\g[f_1]$, $\g[f_2]$ and $\g[f_3]$ are elements of $\soiv[{\permuts[\len]}]$, so $\lambda\g[f_1]\in\soiv[{\permuts[\len]}]$ and $\g[f_2]+\g[f_3]\in\soiv[{\permuts[\len]}]$ by Prop.~\ref{prop:soiv finite exchangeability is coherent}.
Then indeed $\lambda\g[f_1]\in\soiv[\permuts]$ and $\g[f_2]+\g[f_3]\in\soiv[\permuts]$.
\end{proof}

\begin{proof}[Proof of Prop.~\ref{prop:connection countable exchangeability and finite exchangeability}]
The proof for sets of desirable gambles, in a more general context, can be found in~\cite[Proposition~18]{debock2014}.

We give the proof for choice functions.
For necessity, assume that $\cf$ is exchangeable, or equivalently, that $\cf$ is compatible with $\soiv[\permuts]$.
Use Ref.~\cite[Proposition~31]{Vancamp2017} to infer that then, equivalently,
\begin{equation}\label{eq:prop:connection countable exchangeability and finite exchangeability:1}
\group{\forall\tg[h]\in\soiv[\permuts]}
\group{\forall\tos\in\cfdom\group{\gamblesfs}}
\group[\big]{\set{0,\tg[h]}\subseteq\tos
\then
\group{0\in\cf\group{\tos}\iff\tg[h]\in\cf\group{\tos}}}.
\end{equation}
Consider any $\len$ in $\nats$.
We need to prove that then $\cf_{\len}$ is compatible with $\soiv[{\permuts[\len]}]$, or equivalently, that
\begin{equation}\label{eq:prop:connection countable exchangeability and finite exchangeability:2}
\group{\forall\g[h]\in\soiv[{\permuts[\len]}]}
\group{\forall\os\in\cfdom\group{\gambleson[\ps^{\len}]}}
\group[\big]{\set{0,\g[h]}\subseteq\os
\then
\group{0\in\cf_{\len}\group{\os}\iff\g[h]\in\cf_{\len}\group{\os}}}.
\end{equation}
So consider any $\os$ in $\cfdom\group{\gambleson[\ps^{\len}]}$ such that $0\in\os$, and any $\g[h]$ in $\os$.
Then $\os$ is an element of $\cfdom\group{\gamblesfs}$ and $\g[h]$ an element of $\gamblesfs$, so $0\in\cf\group{\os}\iff\g[h]\in\cf\group{\os}$.
Therefore, by marginalisation, $0\in\cf_{\len}\group{\os}\iff\g[h]\in\cf_{\len}\group{\os}$, whence $\cf_{\len}$ is compatible with $\soiv[{\permuts[\len]}]$, and therefore indeed exchangeable.

For sufficiency, assume that $\cf_{\len}$ is exchangeable for every $\len$ in $\nats$---so it satisfies Eq.~\eqref{eq:prop:connection countable exchangeability and finite exchangeability:2} for every $\len$ in $\nats$.
We need to prove that then $\cf$ is exchangeable.
Using Eq.~\eqref{eq:prop:connection countable exchangeability and finite exchangeability:1}, it suffices to consider any $\tos$ in $\cfdom\group{\gamblesfs}$ such that $0\in\tos$, and any $\tg[h]$ in $\tos$, and prove that $0\in\cf\group{\tos}\iff\tg[h]\in\cf\group{\tos}$.
Since $\tos\cup\set{\tg[h]}$ consist of gambles of finite structure, there is some (sufficiently large) $\len$ in $\nats$ for which $\tos\in\cfdom\group{\gambleson[\ps^{\len}]}$ and $\tg[h]\in\gambleson[\ps^{\len}]$.
Then by Eq.~\eqref{eq:prop:connection countable exchangeability and finite exchangeability:2}, $0\in\cf_{\len}\group{\tos}\iff\tg[h]\in\cf_{\len}\group{\tos}$, so $0\in\cf\group{\tos}\iff\tg[h]\in\cf\group{\tos}$, whence $\cf$ is compatible with $\soiv[\permuts]$, and therefore indeed exchangeable.
\end{proof}

\begin{proof}[Proof of Prop.~\ref{prop:bernstein coherence choice functions}]
We only prove sufficiency, since necessity is trivial.
So consider any $\cftoo$ on $\polson$ such that for every $\len$ in $\nats$, $\cftoo_n$ is coherent.
We prove that then $\cftoo$ is coherent.

For Axiom~\ref{coh cf 1: irreflexivity}, consider any $\os$ in $\polson$.
Then every polynomial in $\os$ has a certain degree; let $\len$ be the maximum of those degrees.
Then $\os\in\cfdom\group{\polsond}$, whence indeed $\cftoo\group{\os}=\cftoo_n\group{\os}\neq\emptyset$, since $\cftoo_n$ is coherent.

For Axiom~\ref{coh cf 2: non-triviality}, consider any $\g[h_1]$ and $\g[h_2]$ in $\polson$ such that $\g[h_1]\vo[{\bern[]}]\g[h_2]$.
Then $\g[h_2]-\g[h_1]\in\Posi\group{\cset{\bern}{\cts\in\cvs[\len],n\in\nats}}$.
Let $n_1$ be the degree of $\g[h_1]$ and $n_2$ the degree of $\g[h_2]$, and let $n\coloneqq\max\set{n_1,n_2}$.
Then the degree of the polynomial $\g[h_2]-\g[h_1]$ is not higher than $\len$, so $\g[h_2]-\g[h_1]\in\polsond$ and therefore $\g[h_2]-\g[h_1]\in\Posi\group{\cset{\bern}{\cts\in\cvs[\len]}}$.
That guarantees that $\g[h_1]\vo[{\bern[]}]^n\g[h_2]$, whence indeed $\set{\g[h_2]}=\cftoo_n\group{\set{\g[h_1],\g[h_2]}}=\cftoo\group{\set{\g[h_1],\g[h_2]}}$, since $\cftoo_n$ is coherent.

For Axiom~\ref{coh cf 3}, consider any $\os$, $\os[1]$ and $\os[2]$ in $\cfdom\group{\polson}$, and let $\len$ be the highest degree that appears in $\os\cup\os[1]\cup\os[2]$.
Then $\os$, $\os[1]$ and $\os[2]$ all are elements of $\cfdom\group{\polsond}$.
For Axiom~\ref{coh cf 3a: alpha}, assume that $\cftoo\group{\os[2]}\subseteq\os[2]\setminus\os[1]$ and $\os[1]\subseteq\os[2]\subseteq\os$.
Then $\cftoo_n\group{\os[2]}\subseteq\os[2]\setminus\os[1]$ and therefore, since $\cftoo_n$ is coherent, indeed $\cftoo\group{\os}=\cftoo_n\group{\os}\subseteq\os\setminus\os[1]$.
For Axiom~\ref{coh cf 3b: aizerman}, assume that $\cftoo\group{\os[2]}\subseteq\os[1]$ and $\os\subseteq\os[2]\setminus\os[1]$.
Then $\cftoo_n\group{\os[2]}\subseteq\os[2]\setminus\os[1]$ and therefore, since $\cftoo_n$ is coherent, indeed $\cftoo\group{\os[2]\setminus\os}=\cftoo_n\group{\os[2]\setminus\os}\subseteq\os[1]$.

For Axiom~\ref{coh cf 4}, consider any $\g[h]$ in $\polson$, any $\lambda$ in $\posreals$ and any $\os[1]$ and $\os[2]$ in $\cfdom\group{\polson}$, and let $\len$ be the highest degree that appears in $\set{\g[h]}\cup\os[1]\cup\os[2]$.
Then $\g[h]\in\polsond$ and $\os[1]$ and $\os[2]$ both are elements of $\cfdom\group{\polsond}$.
Assume that $\os[1]\subseteq\cftoo\group{\os[2]}=\cftoo_n\group{\os[2]}$, then, since $\cftoo_n$ is coherent, indeed $\lambda\os[1]\subseteq\cftoo_n\group{\lambda\os[2]}=\cftoo\group{\lambda\os[2]}$ and $\os[1]+\set{\g[h]}\subseteq\cftoo_n\group{\os[2]+\set{\g[h]}}=\cftoo\group{\os[2]+\set{\g[h]}}$.
\end{proof}

\begin{proof}[Proof of Theorem~\ref{theorem:countable exchangeability representation}]
We first prove the representation theorem for choice functions.
That $\cf$ is exchangeable is, by Prop.~\ref{prop:connection countable exchangeability and finite exchangeability}, equivalent to, $\cf_{\len}$ is exchangeable, for every $\len$ in $\nats$.
Therefore, for all $\len$ in $\nats$, by Theorem~\ref{theorem:finite exchangeability representation polynomial}, that is equivalent to
\begin{equation*}
\cf_{\len}\group{\os}
=
\cset{\g\in\os}{\Mnmap[\len]\group{\g}\in\rcf_{\len}\group{\Mnmap[\len]\group{\os}}}
\text{ for all $\os$ in $\cfdom\group{\gambleson[\ps^{\len}]}$,}
\end{equation*}
where $\rcf_{\len}$ is uniquely given by
\begin{equation}\label{eq:countable representation:1}
\rcf_{\len}\group{\Mnmap[\len]\group{\os}}
=
\Mnmap[\len]\group{\cf_{\len}\group{\os}}
\text{ for all $\os$ in $\cfdom\group{\gambleson[\ps^{\len}]}$.}
\end{equation}
Consider any choice function $\rcf$ on $\cfdom\group{\polson}$ such that $\rcf\group{\Mnmap[\len]\group{\os}}=\rcf_{\len}\group{\Mnmap[\len]\group{\os}}$.
Since $\Mnmap[\len]$ is a surjection, we find that then $\rcf\group{\os'}=\rcf_{\len}\group{\os'}$ for all $\os'$ in $\cfdom\group{\polsond}$.
But every $\os'$ in $\cfdom\group{\polsond}$ can be identified with $\os\cap\polsond$ for some $\os$ in $\cfdom\group{\polson}$, so $\rcf\group{\os}=\rcf_{\len}\group{\os\cap\polsond}$ for all $\os$ in $\cfdom\group{\polson}$, where, again, we use the assumption that $\rcf_{\len}\group{\emptyset}=\emptyset$, and thus making $\rcf$ unique.
Then $\rcf\group{\Mnmap[\len]\group{\os}}=\bigcup_{n\in\nats}\rcf_{n}\group{\polsond[\simplex][n]\cap\Mnmap[\len]\group{\os}}=\rcf_{\len}\group{\Mnmap[\len]\group{\os}}$, so indeed
\begin{equation}
\cf_{\len}\group{\os}
=
\cset{\g\in\os}{\Mnmap[\len]\group{\g}\in\rcf\group{\Mnmap[\len]\group{\os}}}
\text{ for all $\os$ in $\cfdom\group{\gambleson[\ps^{\len}]}$,}
\end{equation}
proving the first and second statement.

For the third statement, infer from Prop.~\ref{prop:marginalisation preserves coherence} and Prop.~\ref{prop:coherence marginalisation finitary} that $\cf$ is coherent if and only if $\cf_{\len}$ is coherent for every $\len$ in $\nats$.
Infer from Theorem~\ref{theorem:finite exchangeability representation polynomial} that, for every $\len$ in $\nats$, $\cf_{\len}$ is coherent if and only if $\rcf_{\len}$ on $\polsond$ is coherent.
Infer that, by construction, $\rcf\group{\os}=\rcf_{\len}\group{\os}$ for every $\os$ in $\cfdom\group{\polsond}$, then Prop.~\ref{prop:bernstein coherence choice functions} tells us that indeed $\rcf_{\len}$ is coherent if and only if $\rcf$ on $\polson$ is coherent.

We now turn to the representation theorem for sets of desirable gambles.
That $\sodv$ is exchangeable is, by Prop.~\ref{prop:connection countable exchangeability and finite exchangeability}, equivalent to, $\sodv[\len]$ is exchangeable, for every $\len$ in $\nats$.
Therefore, for any $\len$ in $\nats$, by Theorem~\ref{theorem:finite exchangeability representation polynomial}, that is equivalent to $\sodv[\len]=\bigcup\Mnmapei[\len]\group{\tilde{\sodv}_{\len}}$ where $\tilde{\sodv}_{\len}$ is uniquely given by $\tilde{\sodv}_{\len}=\Mnmap[\len]\group{\sodv[\len]}$.
Infer that $\sodv[\len]=\bigcup\osin[{\sodv[\len]}]$ by the compatibility of $\sodv[\len]$ with $\soiv[{\permuts[\len]}]$.
Using Eq.~\eqref{eq:def Hymap on qsg}, infer that $\osin[{\sodv[\len]}]=\Mnmapei[\len]\group{\Mnmap[\len]\group{\sodv[\len]}}$, whence $\sodv[\len]=\bigcup\Mnmapei[\len]\group{\Mnmap[\len]\group{\sodv[\len]}}$, which is equal to $\Mnmapei[\len]\group{\bigcup_{k\in\nats}\Mnmap[k]\group{\sodv[k]}\cap\polsond}$.
This shows that $\sodv[\len]$ equals $\Mnmapei[\len]\group{\tilde{\sodv}\cap\polsond}$ for some $\tilde{\sodv}$, and that $\tilde{\sodv}$ is exactly given by $\tilde{\sodv}=\bigcup_{\len\in\nats}\Mnmap[\len]\group{\sodv[\len]}=\bigcup_{\len\in\nats}\tilde{\sodv[\len]}$, also proving its uniqueness.
This proves the first and the second statement.

For the third statement, since $\tilde{\sodv}=\bigcup_{\len\in\nats}\Mnmap[\len]\group{\sodv[\len]}$, we see that clearly $\tilde{\sodv}$ is coherent if and only every $\Mnmap[\len]\group{\sodv[\len]}$ is coherent, which is, by Theorem~\ref{theorem:finite exchangeability representation polynomial} equivalent to $\sodv[\len]$ is coherent for every $\len$ in $\nats$.
Now use Props.~\ref{prop:marginalisation preserves coherence} and~\ref{prop:coherence marginalisation finitary} to infer that this is equivalent to $\sodv$ is coherent, proving the third statement.
\end{proof}


\vskip 0.2in
\begin{thebibliography}{14}
\providecommand{\natexlab}[1]{#1}
\providecommand{\url}[1]{\texttt{#1}}
\expandafter\ifx\csname urlstyle\endcsname\relax
  \providecommand{\doi}[1]{doi: #1}\else
  \providecommand{\doi}{doi: \begingroup \urlstyle{rm}\Url}\fi

\bibitem[De~Bock et~al.(2016)De~Bock, Van~Camp, Diniz, and
  de~Cooman]{debock2014}
J.~De~Bock, A.~Van~Camp, M.~A. Diniz, and G.~de~Cooman.
\newblock Representation theorems for partially exchangeable random variables.
\newblock \emph{Fuzzy Sets and Systems}, 284:\penalty0 1--30, 2016.
\newblock \doi{10.1016/j.fss.2014.10.027}.

\bibitem[{de}~Cooman and Miranda(2012)]{cooman2011b}
G.~{de}~Cooman and E.~Miranda.
\newblock Irrelevance and independence for sets of desirable gambles.
\newblock \emph{Journal of Artificial Intelligence Research}, 45:\penalty0
  601--640, 2012.
\newblock \doi{10.1613/jair.3770}.

\bibitem[{d}e Cooman and Quaeghebeur(2012)]{cooman2010}
G.~{d}e Cooman and E.~Quaeghebeur.
\newblock Exchangeability and sets of desirable gambles.
\newblock \emph{International Journal of Approximate Reasoning}, 53\penalty0
  (3):\penalty0 363--395, 2012.
\newblock \doi{10.1016/j.ijar.2010.12.002}.
\newblock Precisely imprecise: A collection of papers dedicated to Henry E.
  Kyburg, Jr.

\bibitem[de~Cooman et~al.(2009)de~Cooman, Quaeghebeur, and
  Miranda]{cooman2006d}
G.~de~Cooman, E.~Quaeghebeur, and E.~Miranda.
\newblock Exchangeable lower previsions.
\newblock \emph{Bernoulli}, 15\penalty0 (3):\penalty0 721--735, 2009.
\newblock \doi{10.3150/09-BEJ182}.
\newblock URL \url{http://hdl.handle.net/1854/LU-498518}.

\bibitem[{d}e Finetti(1937)]{finetti1937}
B.~{d}e Finetti.
\newblock La pr\'evision: ses lois logiques, ses sources subjectives.
\newblock \emph{Annales de l'Institut Henri Poincar\'e}, 7:\penalty0 1--68,
  1937.
\newblock {E}nglish translation in \cite{kyburg1964}.

\bibitem[Johnson et~al.(1997)Johnson, Kotz, and Balakrishnan]{johnson1997}
N.~L. Johnson, S.~Kotz, and N.~Balakrishnan.
\newblock \emph{Discrete Multivariate Distributions}.
\newblock Wiley Series in Probability and Statistics. John Wiley and Sons, New
  York, 1997.

\bibitem[Kadane et~al.(2004)Kadane, Schervish, and Seidenfeld]{Kadane2004}
J.~B. Kadane, M.~J. Schervish, and T.~Seidenfeld.
\newblock A {Rubinesque} theory of decision.
\newblock \emph{Institute of Mathematical Statistics Lecture Notes-Monograph
  Series}, 45:\penalty0 45--55, 2004.
\newblock \doi{10.1214/lnms/1196285378}.
\newblock URL \url{http://www.jstor.org/stable/4356297}.

\bibitem[Kyburg~Jr. and Smokler(1964)]{kyburg1964}
H.~E. Kyburg~Jr. and H.~E. Smokler, editors.
\newblock \emph{Studies in Subjective Probability}.
\newblock Wiley, New York, 1964.
\newblock Second edition (with new material) 1980.

\bibitem[Rubin(1987)]{Rubin1987}
H.~Rubin.
\newblock A weak system of axioms for ``rational" behavior and the
  nonseparability of utility from prior.
\newblock \emph{Statistics \& Risk Modeling}, 5\penalty0 (1-2):\penalty0
  47--58, 1987.
\newblock \doi{10.1524/strm.1987.5.12.47}.

\bibitem[Seidenfeld(1988)]{seidenfeld1988}
T.~Seidenfeld.
\newblock Decision theory without “independence” or without “ordering”.
\newblock \emph{Economics and Philosophy}, 4:\penalty0 267--290, Oct. 1988.
\newblock \doi{10.1017/S0266267100001085}.

\bibitem[Seidenfeld et~al.(2010)Seidenfeld, Schervish, and
  Kadane]{seidenfeld2010}
T.~Seidenfeld, M.~J. Schervish, and J.~B. Kadane.
\newblock Coherent choice functions under uncertainty.
\newblock \emph{Synthese}, 172\penalty0 (1):\penalty0 157--176, 2010.
\newblock \doi{10.1007/s11229-009-9470-7}.

\bibitem[Van~Camp et~al.(2017)Van~Camp, {d}e Cooman, Miranda, and
  Quaeghebeur]{Vancamp2017}
A.~Van~Camp, G.~{d}e Cooman, E.~Miranda, and E.~Quaeghebeur.
\newblock Coherent choice functions, desirability and indifference.
\newblock \emph{Fuzzy sets and systems}, 2017.
\newblock Submitted for publication.

\bibitem[von Neumann and Morgenstern(1972)]{Neumann1944}
J.~von Neumann and O.~Morgenstern.
\newblock \emph{Theory of Games and Economic Behaviour}.
\newblock Princeton University Press, 3rd edition, 1972.
\newblock ISBN 0691041830.

\bibitem[Walley(1991)]{walley1991}
P.~Walley.
\newblock \emph{Statistical Reasoning with Imprecise Probabilities}.
\newblock Chapman and Hall, London, 1991.

\end{thebibliography}
\end{document}